\documentclass[11pt]{article}

\usepackage[final]{acl}

\usepackage{times}
\usepackage{latexsym}

\usepackage[T1]{fontenc}

\usepackage[utf8]{inputenc}

\usepackage{microtype}

\usepackage{inconsolata}

\usepackage{graphicx}

\usepackage{amsmath}
\usepackage{amssymb}
\usepackage{amsthm}
\usepackage{algorithm}
\usepackage{algpseudocode}
\usepackage{booktabs}
\usepackage{arydshln}

\newtheorem{theorem}{Theorem}

\title{ASymPO: Asymmetric-Scale Policy Optimization for Asynchronous LLM Post-Training Without Behavior Information}

\author{Zehua Liu, Yuxuan Yao, Xiaojin Fu, Tao Zhong, Mingxuan Yuan \\
      Huawei Technologies\\ \texttt{liuzehua@connect.hku.hk}}

\begin{document}
\maketitle
\begin{abstract}
Asynchronous reinforcement learning can improve language-model post-training throughput by decoupling response generation from policy optimization, but stale responses introduce distribution drift. Standard behavior-corrected methods control this drift with behavior-policy probabilities, importance ratios, or clipping, which requires token-aligned, versioned, and numerically consistent behavior log-probabilities across rollout and learner systems. We ask whether asynchronous group-relative RL can instead be stabilized using only current-policy probabilities. We identify a scale-imbalance failure mode: when stale responses are evaluated under the current policy, positive and negative loss terms can appear at different negative log-probability scales, so zero-sum advantages no longer imply balanced loss contributions. We propose Asymmetric-Scale Policy Optimization (ASymPO), which normalizes each response's token loss by its current average token negative log-probability. ASymPO requires no behavior-policy probabilities, restores response-level zero-sum balance, and preserves a nonzero learning signal. We also introduce Scaled Policy Optimization (SPO), a fixed negative-scaling baseline, and evaluate both current-policy-only objectives in asynchronous mathematical reasoning post-training.
\end{abstract}

\section{Introduction} \label{sec:introduction}

Reinforcement learning (RL) is increasingly used to improve language models on reward-driven tasks, especially mathematical and reasoning problems where supervised fine-tuning does not directly optimize final success \citep{ouyang_2022_training,shao_2024_deepseekmath,guo_2025_deepseek}. To improve training throughput, modern systems often decouple response generation from policy optimization: rollout workers sample responses from recent policy snapshots while the learner updates a newer policy \citep{zeng_2026_glm}. This asynchronous design is efficient, but it creates distribution drift between the behavior policy that generated a response and the current policy that trains on it.

The standard way to address distribution drift is to use behavior-policy information. PPO-style methods compare $\pi_\theta$ with the behavior policy through importance ratios and clipping \citep{schulman_2017_proximal}, and related mechanisms are used in large-scale post-training systems \citep{guo_2025_deepseek,ma_2025_stabilizing,zeng_2026_glm}. These corrections are principled, but in asynchronous post-training pipelines they make the rollout--learner interface part of the correctness condition. First, each consumed response must carry behavior log-probabilities, or enough information to recompute them, and these quantities must remain aligned with the exact token sequence. Second, the response must be associated with the correct policy snapshot; otherwise the denominator of the importance ratio no longer corresponds to the behavior distribution that generated the data. Third, train--inference numerical drift and rollout staleness must be controlled, typically through recomputation, filtering, or scheduling rules, so that the resulting ratios remain meaningful \citep{team_2025_every,noukhovitch_2024_asynchronous,ma_2025_stabilizing,zeng_2026_glm}.

The naive current-policy objective is attractive for this reason, but it is not automatically stable. In group-relative RL, zero-sum advantages do not by themselves guarantee a balanced loss. In our asynchronous RL training pipeline, we observe that positive and negative advantage responses can differ substantially in their current log-probability scales. In particular, a stale negative-advantage response can have a much larger current negative log-probability scale than positive responses, causing the negative side of the loss to dominate and destabilize training.

We address this problem by scaling the current-policy loss directly, while keeping the training interface restricted to quantities already available on the learner side. The goal is to retain the infrastructure simplicity of current-policy-only training, but prevent a small set of stale negative responses from dominating the update.

We first introduce Scaled Policy Optimization (SPO), a simple method that reduces the influence of negative-advantage responses. Its motivation is direct: the instability is often driven by overly strong updates from stale negative samples, so weakening this side of the loss can make training more stable. SPO is simple and empirically useful, but its scaling rule is manually designed and does not adapt to the actual scale of each response.

We then propose Asymmetric-Scale Policy Optimization (ASymPO), which replaces this manual rule with an adaptive response-level normalization. ASymPO uses the current policy's own response scale to balance positive and negative loss contributions. Responses whose current loss scale is already large are automatically moderated, while the overall update still preserves the direction implied by the group-relative advantages. In this way, ASymPO targets the scale imbalance directly without introducing behavior-policy probabilities into the infrastructure.

This paper makes four contributions.
\textbf{First}, we identify a scale-imbalance failure mode in
current-policy-only asynchronous group-relative RL and show how
behavior correction avoids it through clipping.
\textbf{Second}, we formulate SPO and ASymPO as scaled current-policy
objectives, and prove that ASymPO exactly balances response-level
positive and negative loss contributions under zero-sum advantages.
\textbf{Third}, we show that ASymPO can be deployed with a
\emph{strictly more compact} rollout--learner interface than importance-ratio methods (Table~\ref{tab:interface}): it requires neither behavior
log-probability transport, logit recomputation at training precision,
nor policy-version bookkeeping.
\textbf{Fourth}, we empirically evaluate naive current-policy training, GPG, behavior-corrected GRPO, SPO, and ASymPO on asynchronous
mathematical-reasoning post-training across 3 model families.

\section{Related Works} \label{sec:related_works}

\paragraph{RLHF and preference optimization}
Reinforcement learning from human feedback has been widely used for language-model alignment and post-training \citep{christiano_2017_deep,ziegler_2019_fine,stiennon_2020_learning,bai_2022_training,ouyang_2022_training}. Related preference-optimization methods include direct or reference-free objectives that train from preference data without the full classical RLHF pipeline \citep{rafailov_2023_direct,ethayarajh_2023_kto,hong_2024_orpo}.

\paragraph{Reasoning-oriented post-training}
For mathematical reasoning and verifiable tasks, prior work has studied outcome supervision, process supervision, rejection fine-tuning, and verifier-based training \citep{cobbe_2021_training,uesato_2022_solving,lightman_2023_let,yuan_2023_scaling}. Recent open reasoning models and systems further use RL-style post-training and group-relative objectives for improving reasoning behavior \citep{shao_2024_deepseekmath,guo_2025_deepseek}.

\paragraph{Policy-gradient methods}
Policy-gradient methods provide the optimization basis for many RL post-training algorithms, including REINFORCE, actor-critic methods, TRPO, and PPO \citep{williams_1992_simple,sutton_1999_policy,schulman_2015_trust,schulman_2017_proximal}. Off-policy and lagged-policy training is commonly handled with importance sampling, clipped ratios, trust regions, or return-correction operators \citep{munos_2016_safe,wang_2016_sample,espeholt_2018_impala}.

\paragraph{Distributed and asynchronous reinforcement learning}
Distributed RL systems decouple experience generation and policy optimization through many actors, replay buffers, or actor-learner architectures \citep{nair_2015_massively,mnih_2016_asynchronous,horgan_2018_distributed,espeholt_2018_impala,kapturowski_2018_recurrent}. Related large-model training reports discuss PPO-style language-model RL, REINFORCE-style simplifications, asynchronous rollout generation, training-inference discrepancy, and routing replay \citep{zheng_2023_secrets,ahmadian_2024_back,team_2025_every,ma_2025_stabilizing,zeng_2026_glm,noukhovitch_2024_asynchronous}.

\paragraph{Negative samples in post-training}
Several post-training methods simplify or modify the treatment of preference labels, unsuccessful responses, or negative samples \citep{yuan_2023_scaling,rafailov_2023_direct,ethayarajh_2023_kto,hong_2024_orpo,ahmadian_2024_back,liu_2026_rift}. These works are related to broader efforts to make language-model post-training less dependent on dense value estimation or heavy RL-specific infrastructure.

\section{Instability in Asynchronous RL} \label{sec:instability}

This section formalizes the instability that ASymPO addresses. In standard policy-gradient objectives, distinguishing the current policy $\pi_\theta$ from the behavior policy $\pi_b$ serves a structural role: all responses in a group share the reference scale of $\pi_b$, and clipping bounds the per-token ratio $\rho(\theta)=\pi_\theta/\pi_b$ to prevent the current-policy log-probability scale from drifting far from that reference. Together, the shared reference and the clipping bound balance the loss contributions of positive and negative advantage responses. In asynchronous training that omits $\pi_b$ and optimizes a current-policy-only loss, this scale control is absent. Positive and negative advantage responses can then enter the loss at inconsistent current-policy scales, breaking the cancellation that the zero-sum advantage baseline is designed to provide and causing training instability.

\subsection{Behavior-Corrected Balance}

For a prompt $x\sim\mathcal{D}$, let the behavior policy $\pi_b$ sample a group of responses $\{y_g\}_{g=1}^G$, where $y_g=(a_{g,1},\ldots,a_{g,m_g})$. A reward function assigns $r_g=r(x,y_g)$, and the group-relative advantage is
\begin{equation}
A_g=r_g-\hat{r}, \quad \hat{r}=\frac{1}{G}\sum_{j=1}^G r_j.
\end{equation}
Thus $\sum_{g=1}^G A_g=0$. This zero-sum property is the intended stabilizing signal: positive and negative samples should contribute comparable policy-loss mass in opposite directions.

For a policy $\pi_\mu$, define the average token negative log-probability of response $y_g$ as
\begin{equation}
S_{\mu,g}=-\frac{1}{m_g}\sum_{i=1}^{m_g}\log \pi_\mu(a_{g,i}\mid x,a_{g,<i}).
\end{equation}
The response-level log-probability loss balance is measured by
\begin{equation}
\Delta_\theta=\sum_{g=1}^G A_g S_{\theta,g}.
\end{equation}
If $\Delta_\theta\approx0$, then the positive and negative advantage samples have comparable total loss contributions, not merely comparable raw advantages.

Traditional behavior-corrected methods compare $\pi_\theta$ against $\pi_b$. PPO-style objectives use token-level ratios and clipped ratios
\begin{equation}
\begin{split}
\rho_{g,i}(\theta)&=\frac{\pi_\theta(a_{g,i}\mid x,a_{g,<i})}{\pi_b(a_{g,i}\mid x,a_{g,<i})},\\
\bar{\rho}_{g,i}(\theta)&=\operatorname{clip}\left(\rho_{g,i}(\theta),1-\epsilon,1+\epsilon\right).
\end{split}
\end{equation}
A response-level clipped loss can be written as
\begin{equation}
\begin{split}
\mathcal{L}_{\mathrm{clip}}(\theta)=-\frac{1}{G}\sum_{g=1}^G\frac{1}{m_g}\sum_{i=1}^{m_g}\min\big(&A_g\rho_{g,i}(\theta),\\
&A_g\bar{\rho}_{g,i}(\theta)\big).
\end{split}
\end{equation}
The role of clipping is not only to bound the policy ratio, but also to preserve the balance of the loss contributions induced by normalized advantages. The following theorem makes this statement explicit.

\begin{theorem}[Scale balance under behavior correction] \label{thm:behavior_corrected_balance}
Assume $\sum_{g=1}^G A_g=0$, and let $\mathcal{P}=\{g:A_g>0\}$ and $\mathcal{N}=\{g:A_g<0\}$. Define $B=\sum_{g\in\mathcal{P}}A_g=\sum_{g\in\mathcal{N}}|A_g|$. Suppose the responses are $\delta_b$-balanced under the behavior policy, meaning that there exists $\bar{S}_b$ such that
\begin{equation}
|S_{b,g}-\bar{S}_b|\leq \delta_b \quad \text{for all } g.
\end{equation}
Suppose also that clipping keeps every token ratio used by the update within the trusted range:
\begin{equation}
1-\epsilon \leq \rho_{g,i}(\theta)\leq 1+\epsilon \quad \text{for all } g,i,
\end{equation}
where $0<\epsilon<1$. Let $\tau_\epsilon=\max\{\log(1+\epsilon),-\log(1-\epsilon)\}$. Then
\begin{equation}
\left|\Delta_\theta\right|=\left|\sum_{g=1}^G A_g S_{\theta,g}\right|\leq 2B(\delta_b+\tau_\epsilon).
\end{equation}
Equivalently, the normalized loss $\frac{1}{G}\Delta_\theta$ is close to zero whenever the behavior-policy scale dispersion $\delta_b$ and the clipping width $\epsilon$ are small.
\end{theorem}

Theorem~\ref{thm:behavior_corrected_balance} states that behavior correction transfers balance from advantages to loss contributions. The group baseline gives $\sum_g A_g=0$. The behavior policy supplies a common reference scale for sampled responses through $S_{b,g}$. Clipping then prevents $S_{\theta,g}$ from moving far from that reference. As a result, the positive sum $\sum_{g\in\mathcal{P}}A_gS_{\theta,g}$ and the negative sum $\sum_{g\in\mathcal{N}}|A_g|S_{\theta,g}$ remain close, so the policy loss is not dominated by one side of the group.

\subsection{Naive Loss Without Behavior Policy}

In many asynchronous implementations, storing or recomputing $\pi_b$ is expensive, so the learner may use a naive current-policy-only loss:
\begin{equation} \label{equ:naive_loss}
\mathcal{L}_{\mathrm{naive}}(\theta)=\frac{1}{G}\sum_{g=1}^G A_g S_{\theta,g}.
\end{equation}
Let $\mathcal{P}=\{g:A_g>0\}$ and $\mathcal{N}=\{g:A_g<0\}$. Since $\sum_g A_g=0$, we have $\sum_{g\in\mathcal{P}}A_g=\sum_{g\in\mathcal{N}}|A_g|$, but the loss decomposes as
\begin{equation}
\begin{split}
\mathcal{L}_{\mathrm{naive}}(\theta)=\frac{1}{G}\bigg(&\sum_{g\in\mathcal{P}} A_g S_{\theta,g}\\
&-\sum_{g\in\mathcal{N}} |A_g| S_{\theta,g}\bigg).
\end{split}
\end{equation}
This loss drops the behavior-policy ratios. The behavior-policy scale may still be balanced, because the responses were sampled from $\pi_b$, but the condition that connects $\pi_b$ to $\pi_\theta$ is now absent. Writing
\begin{equation}
d_g=S_{\theta,g}-S_{b,g}=-\frac{1}{m_g}\sum_{i=1}^{m_g}\log \rho_{g,i}(\theta),
\end{equation}
we obtain
\begin{equation}
\sum_{g=1}^G A_g S_{\theta,g}=\sum_{g=1}^G A_g S_{b,g}+\sum_{g=1}^G A_g d_g.
\end{equation}
The first term is small under the behavior-scale assumption in Theorem~\ref{thm:behavior_corrected_balance}. The second term is uncontrolled without clipping or behavior-policy correction. If stale negative-advantage responses have much larger current negative log-probability than they had under $\pi_b$, then $d_g$ is large for $g\in\mathcal{N}$, and the negative part of the loss can dominate:
\begin{equation}
\sum_{g\in\mathcal{N}} |A_g| S_{\theta,g} \gg \sum_{g\in\mathcal{P}} A_g S_{\theta,g}.
\end{equation}
Thus the naive loss can break the positive-negative cancellation even when the raw advantages sum to zero and the sampled responses are well scaled under $\pi_b$. The learner may then keep suppressing stale negative samples while receiving too little compensating signal from positive samples, which can destabilize training and lead to collapse. This is why a current-policy-only method still needs an explicit mechanism for balancing positive and negative loss contributions. ASymPO provides such a mechanism by normalizing each response's current negative log-probability scale, improving stability without requiring access to $\pi_b$.

\section{Methodology} \label{sec:methodology}

Section~\ref{sec:instability} shows that a current-policy-only objective can become unstable when positive and negative advantage responses enter the loss at different current-policy scales. We therefore study a general scaled log-probability objective. For a prompt $x$ and response $y_g=(a_{g,1},\ldots,a_{g,m_g})$, let
\begin{equation}
\begin{split}
p_{g,i}(\theta)&=\pi_\theta(a_{g,i}\mid x,a_{g,<i}),\\
S_{\theta,g}&=-\frac{1}{m_g}\sum_{i=1}^{m_g}\log p_{g,i}(\theta).
\end{split}
\end{equation}
Given group-relative advantages $\{A_g\}_{g=1}^G$, we consider objectives of the form
\begin{equation} \label{eq:scaled_policy_loss}
\begin{split}
\mathcal{L}_{C}(\theta)&=-\frac{1}{G}\sum_{g=1}^G\frac{1}{m_g}\sum_{i=1}^{m_g}A_g C_g\log p_{g,i}(\theta)\\
&=\frac{1}{G}\sum_{g=1}^G A_g C_g S_{\theta,g}.
\end{split}
\end{equation}
The central design question is how to choose the coefficient $C_g$ using only current-policy probabilities. The coefficient should reduce the harmful dominance of stale negative samples while preserving the useful learning signal from positive samples.

\subsection{Scaled Policy Optimization}

Scaled Policy Optimization (SPO) is the simplest instance of Eq.~\eqref{eq:scaled_policy_loss}. Its motivation is direct: if negative-advantage samples can produce disproportionately harmful updates, then their contribution should be reduced by a fixed coefficient. SPO sets
\begin{equation}
C_g^{\mathrm{SPO}}=
\begin{cases}
1, & A_g\geq 0,\\
\alpha, & A_g<0,
\end{cases}
\end{equation}
where $\alpha\in(0,1)$ is a hyperparameter. Equivalently, SPO uses the reweighted advantage $\tilde{A}_g=A_g$ for $A_g\geq0$ and $\tilde{A}_g=\alpha A_g$ for $A_g<0$, giving
\begin{equation}
\mathcal{L}_{\mathrm{SPO}}(\theta)=-\frac{1}{G}\sum_{g=1}^G\frac{1}{m_g}\sum_{i=1}^{m_g}\tilde{A}_g\log p_{g,i}(\theta).
\end{equation}
This design keeps the direction of every update unchanged: positive responses are still reinforced, and negative responses are still suppressed. The only change is that negative responses are suppressed less aggressively. In our experiments, this fixed scaling substantially improves stability, which supports the diagnosis that unbalanced negative loss is a major failure mode.


A key limitation of SPO is its reliance on a manually chosen $\alpha$. A fixed coefficient cannot distinguish mildly stale negatives from those with vanishing current-policy probabilities, nor does it account for the response-dependent scale variations that drive the instability discussed in Section~\ref{sec:instability}. This motivates an adaptive coefficient that dynamically balances loss contributions based on each response’s current scale.

\subsection{Asymmetric-Scale Policy Optimization}

Asymmetric-Scale Policy Optimization (AsymPO) chooses $C_g$ from the current response scale rather than from a fixed sign-dependent rule. Let $\operatorname{sg}(\cdot)$ denote the stop-gradient operator: it has the same forward value as its argument but is treated as a constant during back-propagation. AsymPO sets
\begin{equation}
C_g^{\mathrm{AsymPO}}=\frac{1}{\operatorname{sg}(S_{\theta,g})}
\end{equation}
and optimizes
\begin{equation} \label{eq:sbpo_loss}
\mathcal{L}_{\mathrm{AsymPO}}(\theta)=\frac{1}{G}\sum_{g=1}^G A_g\frac{S_{\theta,g}}{\operatorname{sg}(S_{\theta,g})}.
\end{equation}
Thus every token loss in a response is divided by that response's own average token negative log-probability. Responses with large current negative log-probability receive smaller coefficients, while responses with small current negative log-probability receive larger coefficients. Unlike SPO, this scaling is not selected by a manually tuned sign rule; it is determined by the current policy's response-level scale.

\begin{theorem}[Exact response-level loss balance of AsymPO] \label{thm:sbpo_balance}
Assume $S_{\theta,g}>0$ for all $g$ and $\sum_{g=1}^G A_g=0$. Let $\mathcal{P}=\{g:A_g>0\}$ and $\mathcal{N}=\{g:A_g<0\}$. Under the AsymPO coefficient $C_g^{\mathrm{AsymPO}}=1/\operatorname{sg}(S_{\theta,g})$, the forward response-level loss contributions satisfy
\begin{equation}
\sum_{g\in\mathcal{P}} A_g\frac{S_{\theta,g}}{\operatorname{sg}(S_{\theta,g})}=\sum_{g\in\mathcal{P}}A_g
\end{equation}
and
\begin{equation}
\sum_{g\in\mathcal{N}} |A_g|\frac{S_{\theta,g}}{\operatorname{sg}(S_{\theta,g})}=\sum_{g\in\mathcal{N}}|A_g|.
\end{equation}
Consequently, the signed AsymPO loss is exactly balanced:
\begin{equation}
\sum_{g=1}^G A_g\frac{S_{\theta,g}}{\operatorname{sg}(S_{\theta,g})}=0.
\end{equation}
Moreover, the back-propagated gradient is
\begin{equation}
\nabla_\theta \mathcal{L}_{\mathrm{AsymPO}}(\theta)=\frac{1}{G}\sum_{g=1}^G\frac{A_g}{\operatorname{sg}(S_{\theta,g})}\nabla_\theta S_{\theta,g},
\end{equation}
so each response update is normalized by its own current scale.
\end{theorem}

Theorem~\ref{thm:sbpo_balance} gives the desired property directly. The group baseline balances the raw advantages, and AsymPO makes the response-level loss inherit this balance by normalizing away the current-policy scale $S_{\theta,g}$. At the same time, the stop-gradient normalization preserves a nonzero learning signal: the optimizer still increases probabilities for positive-advantage responses and decreases probabilities for negative-advantage responses, but the magnitude of each response's update is measured relative to its own current scale.

AsymPO therefore addresses the weakness of SPO. SPO reduces negative samples by a fixed human-designed factor $\alpha$, which is effective but not adaptive to the actual source of instability. AsymPO instead balances positive and negative loss contributions automatically, using only the current policy probabilities already required by the naive objective. It does not store, transmit, or recompute behavior-policy probabilities, making it suitable for asynchronous training while simplifying the surrounding infrastructure and directly targeting the scale imbalance identified in Section~\ref{sec:instability}.

\begin{table}[t]
\centering
\fontsize{8}{9}\selectfont
\begin{tabular}{lcc}
\toprule
\textbf{Quantity sent rollout$\to$learner} & \textbf{GRPO} & \textbf{ASymPO} \\
\midrule
Sampled tokens $\{a_{g,i}\}$              & \checkmark & \checkmark \\
Scalar reward $r_g$                       & \checkmark & \checkmark \\
Per-token behavior log-prob $\log\pi_b$   & \checkmark & --- \\
Policy-version tag                        & \checkmark & --- \\
Inference-precision logit recomputation   & required   & --- \\
\bottomrule
\end{tabular}
\caption{What must flow through the rollout--learner interface.
ASymPO removes the per-token behavior log-probability channel and the
policy-version metadata it depends on, eliminating both lossy transport
and train--inference numerical drift.}
\label{tab:interface}
\vspace{-1em}
\end{table}

Table~\ref{tab:interface} makes this concrete by listing the
quantities each method must transport through the rollout--learner
interface, directly addressing the four infrastructure costs
identified in \S\ref{sec:introduction}.

\section{Experiments} \label{sec:experiments}

We evaluate current-policy-only training objectives for mathematical reasoning tasks, where rewards are verifiable and group-relative RL has been widely adopted. The experiments reported in this section focus on the SPO and AsymPO. Our experiments are designed to answer the following research questions:
\begin{itemize}
\item At what accuracy cost, if any, does removing the behavior-log-prob channel from the rollout--learner interface come? 
\item Does the automatically balanced method AsymPO perform better than the artificially scaled method SPO?
\item Compared with other RL algorithms that avoid importance sampling, such as GPG \citep{chu_2025_gpg}, do SPO and AsymPO improve performance or stability in asynchronous training?
\end{itemize}

\subsection{Setup}

\paragraph{Models and Datasets} We conduct experiments on Qwen3-1.7B-Base, Qwen3-4B-Base \citep{yang_2025_qwen3}, and LLaMA-3.2-3B-Instruct to assess mathematical reasoning performance across model families. We train on a randomly sampled subset of $4$k problems from the MATH training set \citep{hendrycks_2021_math}.

\paragraph{Methods} For the baseline, we use naive loss objective \eqref{equ:naive_loss}. We also compare with GRPO, which is a recent importance-sampling-based method that uses behavior-policy probabilities \citep{guo_2025_deepseek}. For current-policy-only methods, we evaluate SPO and AsymPO as proposed in Section~\ref{sec:methodology}. We also include GPG \citep{chu_2025_gpg}, a recent method that avoids importance sampling by using a value-function baseline instead of a group baseline. GPG does not use behavior-policy probabilities, but it also does not have the zero-sum advantage balance that motivates our analysis and methods. Comparing with GPG allows us to test whether the specific scale-balancing mechanism of AsymPO offers benefits beyond the general idea of avoiding importance sampling.

\paragraph{Implementation Details} We implement all RL training with VeRL \citep{sheng_2024_hybridflow}. For asynchronous training, we follow the VeRL configuration and set ppo mini batch size to $32$, train batch size to $512$, and staleness threshold to $0.5$. For each training prompt, we sample $8$ rollouts to compute group-relative advantages. 
For ASymPO and SPO, we modify the VeRL rollout--learner interface to \emph{not} transmit per-token behavior log-probabilities: the learner receives only sampled tokens and scalar rewards, and recomputes every log-probability under $\pi_\theta$. This mirrors the interface in Table~\ref{tab:interface} and makes the implementation faithful to the simplification claim of \S\ref{sec:introduction}; the same VeRL deployment is used for GRPO baselines with the standard behavior-log-prob channel enabled. 
For each problem, the maximum prompt length is $1024$ tokens and the maximum response length is $4096$ tokens. We use a learning rate of $2 \times 10^{-6}$ and train for an equivalent of $3$ epochs for all models. For the SPO, the negative coefficient is set to $\alpha=0.2$, inspired by the work \citep{liu_2026_rift}. Unless otherwise specified, the remaining hyperparameters of each RL algorithm follow the default values recommended in the VeRL documentation.

\paragraph{Evaluation Benchmarks} Following prior work on mathematical reasoning, we evaluate on AIME24 \citep{aime24}, AIME25 \citep{aime25}, MATH500 \citep{lightman_2023_let}, AMC23, GSM8K \citep{cobbe_2021_gsm8k}, and Minerva-Math \citep{lewkowycz_2022_solving}. We use Evalscope \citep{modelscope_2024_evalscope} as the evaluation framework. For each problem, we sample $8$ rollouts. Accuracy is computed by comparing model responses with ground-truth answers, and we report mean@8 and pass@8 performances.

\subsection{Results}

Figure~\ref{fig:training_rewards} shows the training reward curves on Qwen3-1.7B-Base as a representative example, Table~\ref{tab:main_mean8} reports mean@8 accuracy, and Table~\ref{tab:main_pass8} reports pass@8 accuracy. The tables summarize the final benchmark results across the three model families. For all three models, the naive loss and GPG collapse during training, leaving no meaningful final checkpoint for benchmark evaluation.

\begin{figure}[t]
\centering
\includegraphics[width=\linewidth]{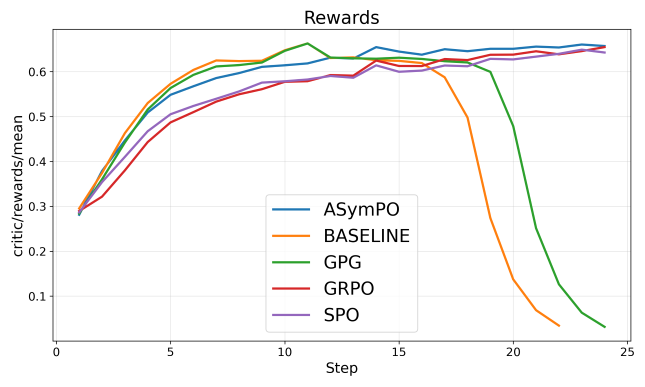}
\caption{Training reward curves for asynchronous RL on Qwen3-1.7B-Base, shown as a representative example. The baseline denotes the naive current-policy loss. The baseline and GPG curves collapse late in training, while GRPO, SPO, and AsymPO maintain stable rewards. The same collapse behavior for GPG and the naive loss was also observed on LLaMA-3.2-3B-Instruct and Qwen3-4B-Base.}
\label{fig:training_rewards}
\vspace{-1em}
\end{figure}

\begin{table*}[t]
\centering
\renewcommand{\arraystretch}{1.12}
\setlength{\dashlinedash}{0.6pt}
\setlength{\dashlinegap}{2pt}
\setlength{\tabcolsep}{3.6pt}
\begin{tabular}{@{}lrrrrrrr@{}}
\toprule
Method & AIME24 & AIME25 & MATH500 & AMC23 & GSM8K & Minerva & Avg. \\
\midrule
\multicolumn{8}{c}{\textit{Qwen3-1.7B-Base}} \\
\hdashline
Naive Loss & \multicolumn{7}{c}{Collapsed during training} \\
GPG & \multicolumn{7}{c}{Collapsed during training} \\
GRPO & 6.67 & \textbf{4.58} & 65.87 & \textbf{37.50} & 81.79 & \textbf{28.26} & \textbf{37.45} \\
SPO & 6.25 & 4.16 & 64.85 & 33.70 & \textbf{83.14} & 26.33 & 36.41 \\
AsymPO & \textbf{7.50} & 4.16 & \textbf{66.20} & 33.70 & 82.87 & 27.44 & 36.98 \\
\midrule
\multicolumn{8}{c}{\textit{LLaMA-3.2-3B-Instruct}} \\
\hdashline
Naive Loss & \multicolumn{7}{c}{Collapsed during training} \\
GPG & \multicolumn{7}{c}{Collapsed during training} \\
GRPO & 10.00 & \textbf{3.33} & 49.37 & 23.37 & \textbf{80.88} & \textbf{18.08} & 30.84 \\
SPO & 8.33 & \textbf{3.33} & 48.53 & 20.92 & 79.56 & 16.96 & 29.61 \\
AsymPO & \textbf{15.42} & \textbf{3.33} & \textbf{49.95} & \textbf{23.91} & 79.48 & \textbf{18.08} & \textbf{31.70} \\
\midrule
\multicolumn{8}{c}{\textit{Qwen3-4B-Base}} \\
\hdashline
Naive Loss & \multicolumn{7}{c}{Collapsed during training} \\
GPG & \multicolumn{7}{c}{Collapsed during training} \\
GRPO & \textbf{13.33} & \textbf{11.25} & \textbf{78.72} & 44.84 & 90.63 & 34.33 & \textbf{45.52} \\
SPO & 12.08 & 8.75 & 77.45 & \textbf{48.37} & 91.09 & 34.93 & 45.45 \\
AsymPO & 12.50 & 9.16 & 77.53 & 44.57 & \textbf{91.15} & \textbf{35.48} & 45.07 \\
\bottomrule
\end{tabular}
\caption{Mean@8 accuracy on mathematical reasoning benchmarks using $8$ evaluation rollouts per problem. The average is computed over all six listed benchmarks. The best value within each model block among evaluated methods with final checkpoints is bolded. ``Collapsed during training'' indicates that the corresponding model-method training run collapsed, leaving no meaningful final checkpoint for evaluation; Figure~\ref{fig:training_rewards} shows the Qwen3-1.7B-Base case as a representative example.}
\label{tab:main_mean8}
\vspace{-1em}
\end{table*}

The training-reward curves in Figure~\ref{fig:training_rewards} provide a direct view of optimization stability on Qwen3-1.7B-Base. The baseline, corresponding to the naive current-policy loss, and GPG both reach high reward early but collapse late in training, with the reward dropping to a near-zero regime and no meaningful final checkpoint for evaluation. In contrast, GRPO, SPO, and AsymPO maintain stable reward trajectories throughout training. The same collapse pattern for the naive loss and GPG was also observed on LLaMA-3.2-3B-Instruct and Qwen3-4B-Base. These curves show that the final benchmark scores in Tables~\ref{tab:main_mean8} and~\ref{tab:main_pass8} should be interpreted together with training stability: current-policy-only training is not stable by default, but SPO and AsymPO prevent the collapse observed in the naive objective and in GPG across the evaluated model families.

\paragraph{RQ1: Cost of dropping behavior-policy probabilities.}
The results show that current-policy-only objectives can be competitive with GRPO when they include explicit scale control, although the relative ranking depends on the model and metric. On Qwen3-1.7B-Base, GRPO obtains the best average score among completed methods, with $37.45$ mean@8 and $54.56$ pass@8. However, AsymPO is close in mean@8, reaching $36.98$, and SPO is close in pass@8, reaching $53.92$. AsymPO also obtains the best mean@8 scores on AIME24 and MATH500, while SPO obtains the best mean@8 score on GSM8K. On LLaMA-3.2-3B-Instruct, AsymPO surpasses GRPO on both aggregate metrics, improving the mean@8 average from $30.84$ to $31.70$ and the pass@8 average from $46.09$ to $49.19$. On Qwen3-4B-Base, GRPO is strongest on the aggregate metrics, with $45.52$ mean@8 and $60.90$ pass@8, but SPO and AsymPO remain close in mean@8, reaching $45.45$ and $45.07$, respectively. AsymPO also obtains the best Qwen3-4B-Base mean@8 scores on GSM8K and Minerva-Math, and ties GRPO on AIME24 pass@8. Overall, behavior-policy probabilities are not strictly necessary for competitive final performance in this setting, but the current-policy-only methods do not uniformly dominate GRPO.

\paragraph{RQ2: AsymPO versus SPO} AsymPO generally improves over the fixed scaling used by SPO on some models, but the advantage is not uniform. On Qwen3-1.7B-Base, AsymPO improves the mean@8 average from $36.41$ to $36.98$, with gains on AIME24, MATH500, and Minerva-Math, but SPO has a slightly higher pass@8 average, $53.92$ versus $53.17$. On LLaMA-3.2-3B-Instruct, the advantage of AsymPO is clearer: AsymPO improves the mean@8 average from $29.61$ to $31.70$ and the pass@8 average from $45.25$ to $49.19$. It also gives the best pass@8 scores on AIME24, MATH500, and AMC23. On Qwen3-4B-Base, the comparison is mixed: SPO has a slightly higher mean@8 average, $45.45$ versus $45.07$, while AsymPO has a higher pass@8 average, $58.87$ versus $58.06$. AsymPO performs best among the current-policy-only methods on Qwen3-4B-Base mean@8 for AIME24, AIME25, MATH500, GSM8K, and Minerva-Math, whereas SPO is stronger on AMC23 mean@8. For Qwen3-4B-Base pass@8, AsymPO's higher average is driven by AIME24 and AIME25, while SPO is stronger on MATH500, AMC23, GSM8K, and Minerva-Math. These results support the main motivation for AsymPO as an adaptive alternative to a manually chosen negative coefficient, while also showing that the fixed coefficient in SPO remains competitive in some regimes.

\paragraph{RQ3: Stability compared with other current-policy-only methods} The comparison with GPG isolates whether avoiding importance sampling is sufficient by itself. GPG does not require behavior-policy probabilities, but it collapses during training in our asynchronous setting, just as the naive current-policy loss does, across all three evaluated model families. SPO and AsymPO also avoid behavior-policy probabilities, yet both maintain stable training rewards and produce evaluable final checkpoints, including on Qwen3-4B-Base. This difference indicates that the key property is not merely removing importance sampling, but preserving the positive-negative balance of group-relative updates under the current policy. The empirical stability of SPO and AsymPO, together with the collapse of GPG and the naive loss, supports the scale-imbalance analysis in Section~\ref{sec:instability} and the response-level balancing mechanism introduced in Section~\ref{sec:methodology}.

\subsection{Ablation Studies}

ASymPO does not introduce any additional tuning hyperparameter beyond the standard training configuration. Its response-level coefficient is determined directly by the current-policy scale $S_{\theta,g}$ through Eq.~\eqref{eq:sbpo_loss}. Therefore, there is no method-specific hyperparameter ablation for ASymPO.

SPO, in contrast, contains the manually chosen negative coefficient $\alpha$. Following RIFT \citep{liu_2026_rift}, we set $\alpha=0.2$ in all experiments. This coefficient has a substantial effect on the resulting objective. When $\alpha=1$, SPO reduces to the naive current-policy loss because positive- and negative-advantage responses are weighted identically; as shown above, this objective collapses during training in our setting. When $\alpha=0$, negative-advantage responses are removed from the objective, making the update similar to RFT-style positive-only fine-tuning, which gives weak empirical performance because it does not explicitly suppress low-reward responses. These two endpoints illustrate that SPO is sensitive to the choice of $\alpha$, and that selecting an appropriate value is important for stable and effective training.

\section{Conclusions}

This work studied current-policy-only optimization for asynchronous RL, where behavior-policy probabilities can correct distribution drift but substantially complicate the rollout--learner interface. We identified a loss-scale imbalance that arises when stale responses are evaluated under the current policy: although group-relative advantages sum to zero, positive and negative responses can have different current negative log-probability scales, causing one side of the policy loss to dominate. To address this failure mode without behavior-policy probabilities, we introduced two current-policy-only objectives: SPO, which reduces the influence of negative-advantage responses with a fixed coefficient, and ASymPO, which adaptively normalizes each response by its own current average token negative log-probability. ASymPO balances response-level positive and negative loss contributions under zero-sum advantages while retaining a meaningful policy-gradient update, providing a simple mechanism for stabilizing asynchronous group-relative RL when storing or recomputing behavior-policy information is undesirable.

\section*{Limitations}

This work focuses on a specific failure mode of current-policy-only asynchronous RL: the imbalance caused by response-level current-policy loss scales. ASymPO is designed to correct this imbalance, but it does not provide a general solution to all forms of distribution drift. In particular, unlike behavior-corrected objectives that explicitly compare $\pi_\theta$ with the behavior policy, ASymPO does not bound the policy ratio between generation and training. When the behavior policy is very far from the current policy, current-policy probabilities alone may not contain enough information to fully characterize the off-policy mismatch.

The proposed methods also make simplifying design choices. SPO relies on a manually chosen negative scaling coefficient, so its performance can depend on the task, model, reward distribution, and degree of staleness. ASymPO removes this fixed coefficient by normalizing each response with its own current loss scale, but this normalization is still a heuristic transformation of the policy objective rather than a complete trust-region method. Its behavior may depend on how advantages are normalized, how response lengths vary, and how extreme current-policy probabilities appear during training.

Our empirical scope is limited to mathematical reasoning tasks with verifiable rewards and group-relative advantages. This setting is important, but it does not cover broader RLHF settings with learned reward models, multi-turn interaction, tool use, or long-horizon agent tasks. Additional evaluation is needed to determine whether the same scale-imbalance diagnosis and the same normalization strategy remain effective across larger models, more diverse data mixtures, different rollout staleness regimes, and production-scale asynchronous training systems.

Future work should study current-policy-only optimization under stronger theoretical and systems perspectives. On the theoretical side, it would be useful to characterize when response-scale normalization approximates behavior correction and when explicit behavior-policy information is unavoidable. On the algorithmic side, promising directions include adaptive variants that combine scale balancing with lightweight drift diagnostics, token-level or response-level safeguards for extreme probabilities, and extensions beyond group-relative advantages. On the systems side, a more complete evaluation should measure not only task accuracy, but also implementation complexity, failure modes, and robustness under realistic asynchronous rollout and training pipelines.


\bibliography{custom}

\cleardoublepage
\appendix

\section{Detailed Proofs} \label{sec:appendix_1}

\begin{proof}[\textbf{Proof of Theorem~\ref{thm:behavior_corrected_balance}}]
For each response,
\begin{equation}
S_{\theta,g}=S_{b,g}-\frac{1}{m_g}\sum_{i=1}^{m_g}\log \rho_{g,i}(\theta).
\end{equation}
The clipping-range assumption implies
\begin{equation}
\left|S_{\theta,g}-S_{b,g}\right|\leq \tau_\epsilon.
\end{equation}
Since $\sum_g A_g=0$, we can subtract the common behavior-policy scale:
\begin{equation}
\Delta_\theta=\sum_{g=1}^G A_g S_{\theta,g}=\sum_{g=1}^G A_g\left(S_{\theta,g}-\bar{S}_b\right).
\end{equation}
Therefore,
\begin{equation}
\begin{split}
\left|\Delta_\theta\right|&\leq \sum_{g=1}^G |A_g||S_{b,g}-\bar{S}_b|\\
&\quad+\sum_{g=1}^G |A_g||S_{\theta,g}-S_{b,g}|\\
&\leq (\delta_b+\tau_\epsilon)\sum_{g=1}^G |A_g|\\
&=2B(\delta_b+\tau_\epsilon).
\end{split}
\end{equation}
\end{proof}

\begin{proof}[Proof of Theorem \ref{thm:sbpo_balance}]
The stop-gradient operator has the same forward value as its argument, so $S_{\theta,g}/\operatorname{sg}(S_{\theta,g})=1$ in the forward computation. The two unsigned identities for $\mathcal{P}$ and $\mathcal{N}$ follow immediately. Since $\sum_g A_g=0$, the signed sum also equals zero. For the gradient, $\operatorname{sg}(S_{\theta,g})$ is treated as constant, hence
\begin{equation}
\nabla_\theta\left(A_g\frac{S_{\theta,g}}{\operatorname{sg}(S_{\theta,g})}\right)=\frac{A_g}{\operatorname{sg}(S_{\theta,g})}\nabla_\theta S_{\theta,g}.
\end{equation}
Summing over $g$ and multiplying by $1/G$ gives the stated gradient.
\end{proof}
\section{Additional Properties of ASymPO} \label{sec:appendix_2}

This appendix records several simple consequences of the ASymPO normalization. These results are not needed for defining the method, but they clarify how ASymPO differs from the naive current-policy objective and why the difference becomes important when positive and negative responses are evaluated at different current-policy scales.

For a fixed prompt group, write $S_g=S_{\theta,g}$ for the current average token negative log-probability of response $y_g$, and assume $S_g>0$ for all $g$. Let $\sum_{g=1}^G A_g=0$, $\mathcal{P}=\{g:A_g>0\}$, $\mathcal{N}=\{g:A_g<0\}$, and $B=\sum_{g\in\mathcal{P}}A_g=\sum_{g\in\mathcal{N}}|A_g|$. Define the response-level naive loss and the forward ASymPO loss by
\begin{equation}
\begin{split}
L_{\mathrm{naive}}&=\frac{1}{G}\sum_{g=1}^G A_g S_g,\\
L_{\mathrm{ASymPO}}&=\frac{1}{G}\sum_{g=1}^G A_g\frac{S_g}{\operatorname{sg}(S_g)}.
\end{split}
\end{equation}
The stop-gradient operator affects back-propagation but not the forward value, so $S_g/\operatorname{sg}(S_g)=1$ in the forward computation.

\begin{theorem}[Exact gap between naive loss and ASymPO loss] \label{thm:appendix_gap}
Under the assumptions above,
\begin{equation}
L_{\mathrm{ASymPO}}=0
\end{equation}
and
\begin{equation}
L_{\mathrm{naive}}-L_{\mathrm{ASymPO}}=\frac{B}{G}\left(\bar{S}_{\mathcal{P}}-\bar{S}_{\mathcal{N}}\right),
\end{equation}
where
\begin{equation}
\bar{S}_{\mathcal{P}}=\frac{1}{B}\sum_{g\in\mathcal{P}}A_gS_g, \quad \bar{S}_{\mathcal{N}}=\frac{1}{B}\sum_{g\in\mathcal{N}}|A_g|S_g.
\end{equation}
\end{theorem}

\begin{proof}
Since $S_g/\operatorname{sg}(S_g)=1$ in the forward computation and $\sum_g A_g=0$, we have $L_{\mathrm{ASymPO}}=\frac{1}{G}\sum_g A_g=0$. For the naive loss,
\begin{equation}
\begin{split}
L_{\mathrm{naive}}&=\frac{1}{G}\left(\sum_{g\in\mathcal{P}}A_gS_g-\sum_{g\in\mathcal{N}}|A_g|S_g\right)\\
&=\frac{B}{G}\left(\bar{S}_{\mathcal{P}}-\bar{S}_{\mathcal{N}}\right).
\end{split}
\end{equation}
Combining the two identities gives the stated gap.
\end{proof}

Theorem~\ref{thm:appendix_gap} shows that the difference between the naive loss and ASymPO is exactly the signed scale imbalance between positive and negative responses. If the weighted average current scale of negative responses is larger than that of positive responses, then the naive loss is shifted toward the negative side. ASymPO removes this forward imbalance by normalizing each response to a common response-level scale.

\begin{theorem}[Positive-negative contribution ratio] \label{thm:appendix_ratio}
Assume $\bar{S}_{\mathcal{P}}>0$ and $\bar{S}_{\mathcal{N}}>0$. The ratio between the unsigned negative and positive response-level contributions of the naive loss is
\begin{equation}
R_{\mathrm{naive}}=\frac{\sum_{g\in\mathcal{N}}|A_g|S_g}{\sum_{g\in\mathcal{P}}A_gS_g}=\frac{\bar{S}_{\mathcal{N}}}{\bar{S}_{\mathcal{P}}},
\end{equation}
whereas the corresponding ASymPO ratio is
\begin{equation}
R_{\mathrm{ASymPO}}=\frac{\sum_{g\in\mathcal{N}}|A_g|S_g/\operatorname{sg}(S_g)}{\sum_{g\in\mathcal{P}}A_gS_g/\operatorname{sg}(S_g)}=1.
\end{equation}
\end{theorem}

\begin{proof}
The identity for $R_{\mathrm{naive}}$ follows by substituting the definitions of $\bar{S}_{\mathcal{P}}$ and $\bar{S}_{\mathcal{N}}$. For ASymPO, $S_g/\operatorname{sg}(S_g)=1$, so the numerator becomes $\sum_{g\in\mathcal{N}}|A_g|=B$ and the denominator becomes $\sum_{g\in\mathcal{P}}A_g=B$.
\end{proof}

Theorem~\ref{thm:appendix_ratio} makes the balancing effect explicit. In the naive loss, the relative influence of negative and positive responses is not determined only by the advantages; it is multiplied by the ratio of their current response scales. If negative responses have larger current scales, then the negative side dominates. ASymPO fixes this ratio at one at the response-loss level, so the group-relative advantage balance is preserved after the loss transformation.

\begin{theorem}[Gap bound under bounded scale dispersion] \label{thm:appendix_dispersion}
Suppose there exists a scalar $\bar{S}>0$ and a constant $\delta\geq0$ such that
\begin{equation}
|S_g-\bar{S}|\leq \delta \quad \text{for all } g.
\end{equation}
Then
\begin{equation}
|L_{\mathrm{naive}}-L_{\mathrm{ASymPO}}|\leq \frac{2B\delta}{G}.
\end{equation}
\end{theorem}

\begin{proof}
Since $L_{\mathrm{ASymPO}}=0$ and $\sum_g A_g=0$,
\begin{equation}
L_{\mathrm{naive}}=\frac{1}{G}\sum_{g=1}^G A_g(S_g-\bar{S}).
\end{equation}
Therefore,
\begin{equation}
\begin{split}
|L_{\mathrm{naive}}-L_{\mathrm{ASymPO}}|&=|L_{\mathrm{naive}}|\\
&\leq \frac{1}{G}\sum_{g=1}^G |A_g||S_g-\bar{S}|\\
&\leq \frac{\delta}{G}\sum_{g=1}^G |A_g|\\
&=\frac{2B\delta}{G}.
\end{split}
\end{equation}
\end{proof}

Theorem~\ref{thm:appendix_dispersion} shows that ASymPO remains close to the naive objective when the current response scales are already balanced. The gap grows only with the within-group scale dispersion. Thus ASymPO mainly changes the objective in precisely the regime where the naive loss is unreliable: when responses in the same group have very different current negative log-probability scales.

\begin{theorem}[ASymPO as response-scale normalization] \label{thm:appendix_normalization}
For each response $g$, the forward response-level loss scale induced by ASymPO is independent of $S_g$:
\begin{equation}
\left|A_g\frac{S_g}{\operatorname{sg}(S_g)}\right|=|A_g|.
\end{equation}
For the naive loss, the corresponding response-level scale is $|A_g|S_g$.
\end{theorem}

\begin{proof}
The ASymPO identity follows from $S_g/\operatorname{sg}(S_g)=1$. The naive response-level contribution is $A_gS_g$, whose absolute value is $|A_g|S_g$ because $S_g>0$.
\end{proof}

Theorem~\ref{thm:appendix_normalization} states the basic mechanism behind ASymPO. The naive loss allows the current response scale $S_g$ to multiply the advantage magnitude. ASymPO removes this multiplier in the forward response-level loss, while the stop-gradient denominator still allows gradients to flow through $S_g$. This is why ASymPO can reduce scale-driven dominance without discarding the policy-gradient signal.

\section{Complete ASymPO Algorithm and Empirical Stabilization} \label{sec:appendix_3}

This appendix gives the complete learner-side ASymPO algorithm and discusses practical stabilizers used in empirical implementations. ASymPO is intended for asynchronous group-relative RL where rollout workers may sample responses from a stale behavior policy, but the learner uses only current-policy probabilities. The behavior policy is used to generate text, while its token probabilities are neither stored nor used by the ASymPO update.

For a prompt $x$, let the rollout system produce a group of responses $\{y_g\}_{g=1}^G$, where $y_g=(a_{g,1},\ldots,a_{g,m_g})$. Let $r_g=r(x,y_g)$ be the reward and define the group-relative advantage
\begin{equation}
A_g=r_g-\frac{1}{G}\sum_{j=1}^G r_j.
\end{equation}
For the current learner policy $\pi_\theta$, define
\begin{equation}
\begin{aligned}
p_{g,i}(\theta)=\pi_\theta(a_{g,i}\mid x,a_{g,<i}), \\
S_{\theta,g}=-\frac{1}{m_g}\sum_{i=1}^{m_g}\log p_{g,i}(\theta).
\end{aligned}
\end{equation}
ASymPO minimizes
\begin{equation} \label{eq:appendix_sbpo_loss}
\mathcal{L}_{\mathrm{ASymPO}}(\theta)=\frac{1}{G}\sum_{g=1}^G A_g\frac{S_{\theta,g}}{\operatorname{sg}(S_{\theta,g})},
\end{equation}
where $\operatorname{sg}(\cdot)$ denotes the stop-gradient operator. The forward value of each normalized response scale is one, so the response-level loss inherits the zero-sum balance of the advantages. During back-propagation, however, the denominator is treated as a constant, and the update direction remains the usual policy-gradient direction: positive-advantage responses are reinforced and negative-advantage responses are suppressed.

\begin{algorithm}[t]
\caption{Scale-Balanced Policy Optimization for one learner update}
\label{alg:sbpo_complete}
\begin{algorithmic}[1]
\Require Current policy $\pi_\theta$, prompt batch $\mathcal{B}$, group size $G$, reward function $r$, learning rate $\eta$
\State $\mathcal{L}\gets 0$
\ForAll{$x\in\mathcal{B}$}
\State Sample or receive responses $\{y_g=(a_{g,1},\ldots,a_{g,m_g})\}_{g=1}^G$ from rollout workers
\For{$g=1,\ldots,G$}
\State $r_g\gets r(x,y_g)$
\EndFor
\State $\hat{r}\gets G^{-1}\sum_{j=1}^G r_j$
\For{$g=1,\ldots,G$}
\State $A_g\gets r_g-\hat{r}$
\For{$i=1,\ldots,m_g$}
\State $p_{g,i}(\theta)\gets \pi_\theta(a_{g,i}\mid x,a_{g,<i})$
\EndFor
\State $S_{\theta,g}\gets -m_g^{-1}\sum_{i=1}^{m_g}\log p_{g,i}(\theta)$
\State $\mathcal{L}\gets \mathcal{L}+|\mathcal{B}|^{-1}G^{-1}A_g S_{\theta,g}/\operatorname{sg}(S_{\theta,g})$
\EndFor
\EndFor
\State $\theta\gets \theta-\eta\nabla_\theta\mathcal{L}$
\State \Return updated policy $\pi_\theta$
\end{algorithmic}
\end{algorithm}

In practice, Eq.~\eqref{eq:appendix_sbpo_loss} is implemented at the token level. While the per-response normalization balances the loss scale across responses, it does not constrain individual token probabilities, which can cause numerical instability in two regimes. For $A_g<0$, the optimizer is asked to reduce token probabilities; if a token already has $p_{g,i}(\theta)\ll 1$, the factor $\partial(-\log p)/\partial p=-1/p$ grows without bound, and that single token can dominate the gradient. For $A_g\geq 0$, tokens whose probability already approaches~1 still receive gradient updates pushing them higher, contributing to overfitting without meaningfully improving the response.

A simple stabilizer is to clip each token's log-probability before it enters the ASymPO reduction, with the clipping direction determined by the sign of the response advantage. Define bounds $p_{\mathrm{low}},p_{\mathrm{high}}\in(0,1)$. For a negative-advantage response, tokens with probability below $p_{\mathrm{low}}$ are already sufficiently suppressed and should not be driven further toward zero. For a positive-advantage response, tokens with probability above $p_{\mathrm{high}}$ are already sufficiently reinforced and should not be pushed further toward one. Concretely,
\begin{equation}
\hat{\ell}_{g,i}(\theta)=
\begin{cases}
\operatorname{clip}\!\bigl(\log p_{g,i}(\theta),\;-\infty,\;\log p_{\mathrm{high}}\bigr), & A_g\geq 0,\\[6pt]
\operatorname{clip}\!\bigl(\log p_{g,i}(\theta),\;\log p_{\mathrm{low}},\;+\infty\bigr), & A_g<0.
\end{cases}
\end{equation}
The clipped response scale is then
\begin{equation}
\hat{S}_{\theta,g}=-\frac{1}{m_g}\sum_{i=1}^{m_g}\hat{\ell}_{g,i}(\theta).
\end{equation}
Replacing $S_{\theta,g}$ with $\hat{S}_{\theta,g}$ in Eq.~\eqref{eq:appendix_sbpo_loss} gives the stabilized objective
\begin{equation} \label{eq:appendix_sbpo_stable}
\mathcal{L}_{\mathrm{ASymPO\text{-}stable}}(\theta)=\frac{1}{G}\sum_{g=1}^G A_g\frac{\hat{S}_{\theta,g}}{\operatorname{sg}(\hat{S}_{\theta,g})}.
\end{equation}

The clipping operates differently for the two cases. For $A_g\geq 0$, clipping $\log p$ from above at $\log p_{\mathrm{high}}$ means that once a token's probability exceeds $p_{\mathrm{high}}$, its log-probability is capped; the gradient through that token is identically zero because the clip sits at the boundary, and the optimizer stops reinforcing it. For $A_g<0$, clipping $\log p$ from below at $\log p_{\mathrm{low}}$ means that once a token's probability falls below $p_{\mathrm{low}}$, its log-probability is similarly capped and its gradient is zero, shielding the update from the $1/p$ explosion. In both cases, the token still participates in the response-scale average $\hat{S}_{\theta,g}$, so the per-response normalization structure of ASymPO is preserved; the only tokens that contribute gradient are those whose probabilities lie inside the unclipped region.

This clipping scheme is simpler than the masking heuristic used in earlier ASymPO stabilization experiments. Masking removes confident tokens from both the numerator and denominator of the response scale, which breaks exact per-response balance. Clipping, in contrast, keeps all tokens in the reduction and produces zero gradient naturally at the boundary of the feasible region. The bounds should be chosen conservatively: typical values are $p_{\mathrm{low}}\in[0.01,0.05]$ and $p_{\mathrm{high}}\in[0.95,0.99]$. Monitoring the fraction of tokens that hit each clip boundary provides a useful diagnostic for whether the bounds are set too aggressively.

Other standard safeguards remain compatible. The denominator can be floored via $\operatorname{sg}(\max\{\hat{S}_{\theta,g},\epsilon\})$ to avoid large multipliers when most tokens in a response are clipped. Global gradient clipping is still recommended, since per-token log-probability clipping bounds the loss scale but does not constrain parameter-space gradients. A lightweight KL or entropy regularizer can be added when further control of policy drift is desired, although this reintroduces an explicit regularization target and should be kept conceptually separate from the clipping stabilizer described here.
\section{Supplementary Experiments} \label{sec:appendix_4}

\begin{table*}[t]
\centering
\renewcommand{\arraystretch}{1.12}
\setlength{\dashlinedash}{0.6pt}
\setlength{\dashlinegap}{2pt}
\setlength{\tabcolsep}{3.6pt}
\begin{tabular}{@{}lrrrrrrr@{}}
\toprule
Method & AIME24 & AIME25 & MATH500 & AMC23 & GSM8K & Minerva & Avg. \\
\midrule
\multicolumn{8}{c}{\textit{Qwen3-1.7B-Base}} \\
\hdashline
Naive Loss & \multicolumn{7}{c}{Collapsed during training} \\
GPG & \multicolumn{7}{c}{Collapsed during training} \\
GRPO & \textbf{16.67} & \textbf{23.33} & \textbf{85.60} & \textbf{65.22} & 94.24 & 42.28 & \textbf{54.56} \\
SPO & \textbf{16.67} & \textbf{23.33} & 84.00 & 63.04 & \textbf{94.92} & 41.54 & 53.92 \\
AsymPO & \textbf{16.67} & \textbf{23.33} & 83.60 & 58.70 & 94.09 & \textbf{42.65} & 53.17 \\
\midrule
\multicolumn{8}{c}{\textit{LLaMA-3.2-3B-Instruct}} \\
\hdashline
Naive Loss & \multicolumn{7}{c}{Collapsed during training} \\
GPG & \multicolumn{7}{c}{Collapsed during training} \\
GRPO & 23.33 & 6.67 & 70.99 & 47.83 & \textbf{93.56} & \textbf{34.18} & 46.09 \\
SPO & 26.67 & \textbf{7.08} & 67.39 & 45.65 & 92.34 & 32.35 & 45.25 \\
AsymPO & \textbf{33.33} & 6.67 & \textbf{73.39} & \textbf{54.35} & 93.48 & 33.94 & \textbf{49.19} \\
\midrule
\multicolumn{8}{c}{\textit{Qwen3-4B-Base}} \\
\hdashline
Naive Loss & \multicolumn{7}{c}{Collapsed during training} \\
GPG & \multicolumn{7}{c}{Collapsed during training} \\
GRPO & \textbf{30.00} & \textbf{26.67} & \textbf{90.40} & \textbf{73.91} & \textbf{97.73} & 46.69 & \textbf{60.90} \\
SPO & 23.33 & 16.67 & 89.60 & \textbf{73.91} & 96.89 & \textbf{47.97} & 58.06 \\
AsymPO & \textbf{30.00} & 20.00 & 88.60 & 73.04 & 96.74 & 44.85 & 58.87 \\
\bottomrule
\end{tabular}
\caption{Pass@8 accuracy on mathematical reasoning benchmarks using $8$ evaluation rollouts per problem. The average is computed over all six listed benchmarks. The best value within each model block among evaluated methods with final checkpoints is bolded. ``Collapsed during training'' indicates that the corresponding model-method training run collapsed, leaving no meaningful final checkpoint for evaluation; Figure~\ref{fig:training_rewards} shows the Qwen3-1.7B-Base case as a representative example.}
\label{tab:main_pass8}
\end{table*}

We conduct an additional experiment using DAPO-Math-17K \citep{yu_2026_dapo} as the training-data source. The training set is constructed by randomly sampling $4$k examples from DAPO-Math-17K. Unless otherwise specified, all training and evaluation settings follow the main experimental setup in Section~\ref{sec:experiments}, including the asynchronous RL configuration, rollout count, sequence-length limits, learning rate, training duration, and evaluation protocol. This supplementary experiment is conducted only for Qwen3-1.7B-Base.

Table~\ref{tab:dapo17k_qwen3_17b} reports the final benchmark results. The overall pattern is consistent with the main experiments in Section~\ref{sec:experiments}. The naive current-policy loss and GPG again collapse during asynchronous training, while GRPO, SPO, and AsymPO produce stable final checkpoints. This shows that the collapse of unbalanced current-policy training is not specific to the MATH training subset used in the main experiments, and that explicit scale control remains important when the training data source is changed.

Among the completed methods, the current-policy-only scaled objectives are competitive with, and in this experiment stronger than, GRPO on the aggregate metrics. For mean@8, AsymPO obtains the best average score, improving over GRPO from $37.18$ to $38.66$ and over SPO from $38.13$ to $38.66$. Its gains are broad across benchmarks, with the best scores on AIME24, MATH500, AMC23, GSM8K, and Minerva-Math, while GRPO and SPO tie for the best score on AIME25. For pass@8, SPO gives the highest average score, $57.53$, followed by AsymPO at $56.83$ and GRPO at $54.70$. The benchmark-level results are mixed: SPO is strongest on AIME24 and ties GRPO on AIME25, while AsymPO is strongest on MATH500, GSM8K, Minerva-Math, and ties SPO on AMC23. These results reinforce the main conclusion that behavior-policy probabilities are not required for stable and competitive asynchronous RL, but they also show that the relative ranking between SPO and AsymPO can depend on the metric and data source.

\begin{table*}[t]
\centering
\renewcommand{\arraystretch}{1.12}
\setlength{\dashlinedash}{0.6pt}
\setlength{\dashlinegap}{2pt}
\setlength{\tabcolsep}{3.6pt}
\begin{tabular}{@{}lrrrrrrr@{}}
\toprule
Method & AIME24 & AIME25 & MATH500 & AMC23 & GSM8K & Minerva & Avg. \\
\midrule
\multicolumn{8}{c}{\textit{Mean@8}} \\
\hdashline
Naive Loss & \multicolumn{7}{c}{Collapsed during training} \\
GPG & \multicolumn{7}{c}{Collapsed during training} \\
GRPO & 6.25 & \textbf{4.58} & 63.55 & 39.37 & 81.77 & 27.53 & 37.18 \\
SPO & 8.75 & \textbf{4.58} & 65 & 38.75 & 81.71 & 30.0 & 38.13 \\
AsymPO & \textbf{9.58} & 3.33 & \textbf{66.92} & \textbf{39.69} & \textbf{81.79} & \textbf{30.65} & \textbf{38.66} \\
\midrule
\multicolumn{8}{c}{\textit{Pass@8}} \\
\hdashline
Naive Loss & \multicolumn{7}{c}{Collapsed during training} \\
GPG & \multicolumn{7}{c}{Collapsed during training} \\
GRPO & 16.67 & \textbf{16.67} & 82.8 & 65 & 94.47 & 52.57 & 54.7 \\
SPO & \textbf{30} & \textbf{16.67} & 84.2 & \textbf{67.5} & 94.62 & 52.21 & \textbf{57.53} \\
AsymPO & 26.67 & 13.33 & \textbf{85.4} & \textbf{67.5} & \textbf{95.15} & \textbf{52.94} & 56.83 \\
\bottomrule
\end{tabular}
\caption{Mean@8 and Pass@8 accuracy on mathematical reasoning benchmarks for Qwen3-1.7B-Base trained on a randomly sampled 4k subset of DAPO-17K. All other experimental settings follow Section~\ref{sec:experiments}.}
\label{tab:dapo17k_qwen3_17b}
\end{table*}

\end{document}